\pdfoutput=1
\documentclass{neus2025}
\usepackage{booktabs}
\usepackage{wrapfig}
\usepackage{algorithm}
\usepackage{algorithmic}
\usepackage{booktabs}
\usepackage{multirow}
\usepackage{amsmath}

\title[Logic-Guided Vector Fields for Constrained Generative Modeling]{Logic-Guided Vector Fields for Constrained Generative Modeling}
\usepackage{times}

\author{%
 \Name{Ali Baheri} \Email{akbeme@rit.edu}\\
 \addr Rochester Institute of Technology, Rochester, NY
}

\begin{document}

\maketitle

\vspace{-10 mm}
\begin{abstract}
Neuro-symbolic systems aim to combine the expressive structure of symbolic logic with the flexibility of neural learning; yet, generative models typically lack mechanisms to \emph{enforce} declarative constraints at generation time. We propose Logic-Guided Vector Fields (LGVF), a neuro-symbolic framework that injects symbolic knowledge, specified as differentiable relaxations of logical constraints, into flow matching generative models. LGVF couples two complementary mechanisms: (1) a \emph{training-time} logic loss that penalizes constraint violations along continuous flow trajectories, with weights that emphasize correctness near the target distribution; and (2) an \emph{inference-time} adjustment that steers sampling using constraint gradients, acting as a lightweight, logic-informed correction to the learned dynamics. We evaluate LGVF on three constrained generation case studies spanning linear, nonlinear, and multi-region feasibility constraints. Across all settings, LGVF reduces constraint violations by 59–82\% compared to standard flow matching and achieves the lowest violation rates in each case. In the linear and ring settings, LGVF also improves distributional fidelity as measured by MMD. while in the multi-obstacle setting, we observe a satisfaction–fidelity trade-off, with improved feasibility but increased MMD. Beyond quantitative gains, LGVF yields constraint-aware vector fields exhibiting emergent obstacle-avoidance behavior, routing samples around forbidden regions without explicit path planning. 

\end{abstract}



\section{Introduction}
Generative models are increasingly used as proposal mechanisms inside larger systems rather than as standalone density estimators. In robotics and planning, they generate candidate behaviors that can be evaluated or refined at test time \citep{janner2022planning,chi2023diffusion_policy}. In design and discovery settings, they propose candidates that must satisfy domain rules, where invalid outputs can waste evaluation budgets or break downstream pipelines \citep{brookes2019conditioning}. In all these applications, sample quality alone is not enough: what matters is whether a model produces \emph{valid} outputs under explicit requirements and does so reliably. Despite this need, most generative modeling pipelines remain largely \emph{constraint-blind}. Constraints are typically handled downstream via rejection sampling, projection, or task-specific constrained samplers. These workarounds are often expensive or brittle: rejection becomes impractical when feasible regions are small \citep{robert2004montecarlo}; projection can distort samples and behaves poorly when constraints are nonconvex or disconnected \citep{boyd2004convex}; and guidance-style methods provide heuristic control signals but do not, in general, guaranty satisfaction \citep{dhariwal2021diffusion}. As a result, even when the training data are fully valid, learned generators can still produce invalid samples: the rare failures that matter most in safety-critical settings.

This gap is particularly visible for continuous-time generative models. Diffusion and score-based models can be interpreted as transporting noise into data through a time-evolving denoising process \citep{ho2020denoising,song2021scorebased}. Flow-based formulations make this dynamical view explicit: generation corresponds to integrating learned continuous-time dynamics \citep{chen2018neuralode,grathwohl2019ffjord}, and flow matching offers a scalable objective for learning such dynamics \citep{lipman2023flow}, with closely related ``straightening'' variants such as rectified flow \citep{liu2023flowstraight}. However, standard training remains agnostic to feasibility. The learned transport can route probability mass through forbidden regions during generation, producing invalid samples even when the target distribution lies entirely in the valid set. Enforcing constraints only at the endpoint is therefore insufficient when the generation \emph{trajectory} itself crosses invalid regions, and naively adding penalties during training can be unreliable under complex multi-region constraints.

In this work, we study the following problem: \emph{How can we incorporate logical and geometric constraints into flow-matching generative models so that validity improves without sacrificing sample quality?} We propose LGVF, a lightweight neuro-symbolic framework that integrates constraints directly into the generative dynamics. LGVF assumes that constraints are specified as a logical predicate along with a differentiable measure of violation. It then enforces constraints using two complementary mechanisms. First, during training, LGVF encourages the learned dynamics to follow \emph{constraint-aware transport paths} by penalizing violations \emph{along the generation trajectory}, rather than only checking validity at the final sample. Second, during sampling, LGVF applies a small, scheduled \emph{logic adjustment} that steers trajectories away from violations, acting as a robust last-mile correction when training-time guidance alone is insufficient. This two-stage design couples global shaping of transport with local correction, yielding consistent improvements in validity across constraint types. 



\noindent \textbf{Contributions.} We show that unconstrained flow-based transport can traverse forbidden regions even when the data distribution lies entirely in the valid set. To address this, we propose \emph{Logic-Guided Vector Fields} (LGVF), which integrates symbolic constraints into flow matching via a trajectory-level logic loss during training and a lightweight gradient-based adjustment during sampling. Across linear, nonlinear, and multi-region constraints, LGVF consistently improves validity with minimal impact on sample quality and yields emergent constraint-aware transport (e.g., obstacle avoidance) in the learned dynamics.

\noindent \textbf{Paper Organization.} Section~\ref{sec:related} reviews related work. Section~\ref{sec:method} presents the LGVF methodology, including training objectives, sampling procedures, and constraint formulations. Section~\ref{sec:experiments} provides an experimental evaluation of three case studies. Section~\ref{sec:conclusion} concludes with a discussion and future directions.

\vspace{-6 mm}
\section{Related Work}
\label{sec:related}

\noindent \textbf{Flow-Based Generative Models.} Continuous normalizing flows (CNFs) \citep{chen2018neuralode} model generation as integrating a learned neural ODE, enabling exact likelihoods via the instantaneous change-of-variables formula, but at the cost of expensive ODE solves during training. Flow matching \citep{lipman2023flow, liu2023flowstraight} sidesteps this by regressing onto conditional vector fields along simple interpolation paths, achieving strong sample quality without simulation in the training loop. Rectified flows \citep{liu2022rectified} further encourage near-straight transport for few-step sampling, and optimal transport flow matching \citep{tong2024improving} improves efficiency by using minibatch optimal transport couplings. Closest to our setting, \citet{baheri2025metriplectic} extends conditional flow matching to dissipative dynamics via a metriplectic construction, highlighting how structural priors can be imposed directly on learned vector fields. We build on conditional flow matching and similarly modify the learned dynamics, but our focus is \emph{constraint satisfaction}: we incorporate differentiable logical constraints through a training-time penalty and a lightweight sampling-time correction.

\noindent \textbf{Guided Generation.}
Guidance methods steer generative processes by adding gradients of auxiliary objectives. In diffusion models, classifier guidance \citep{dhariwal2021diffusion} injects classifier gradients into the score to enable high-quality conditional generation, while classifier-free guidance \citep{ho2022classifierfree} achieves similar control without a separate classifier and has become standard in text-to-image systems \citep{rombach2022latent, saharia2022photorealistic}. Universal guidance \citep{bansal2023universal} generalizes this principle to arbitrary differentiable objectives. LGVF follows the same high-level idea of gradient-based steering but targets \emph{hard} constraint satisfaction rather than soft conditional preferences and operates on flow-matching vector fields. From a safety perspective, our inference-time adjustment is analogous in spirit to gradient-based ``safety shaping'' used in decision-making: for instance, \citet{yifru2024concurrent} learning policies jointly with unknown safety constraints, and \citet{baheri2025distributionally} enforcing safety and stability under uncertainty via Lyapunov–barrier objectives. LGVF adapts this notion of constraint-driven correction to generative transport dynamics.


\noindent \textbf{Flow-Based Generative Models.}
Continuous normalizing flows (CNFs) \citep{chen2018neuralode} parameterize generative models as neural ODEs, enabling exact likelihood computation through the instantaneous change of variables formula. However, CNFs require expensive ODE integration during training. Flow matching \citep{lipman2023flow, liu2023flowstraight} addresses this limitation by regressing onto conditional vector fields along simple interpolation paths, avoiding simulation during training while achieving state-of-the-art sample quality. Rectified flows \citep{liu2022rectified} further simplify learned trajectories for few-step generation, and optimal transport flow matching \citep{tong2024improving} improves efficiency using minibatch optimal transport couplings. Our work builds on conditional flow matching, augmenting it with constraint-aware training and sampling.

\noindent \textbf{Guided Generation.}
Classifier guidance \citep{dhariwal2021diffusion} steers diffusion models toward desired classes by adding classifier gradients to the score function, enabling high-quality conditional generation. Classifier-free guidance \citep{ho2022classifierfree} achieves similar effects without a separate classifier by jointly training conditional and unconditional models. These approaches have become standard in text-to-image systems \citep{rombach2022latent, saharia2022photorealistic}. Universal guidance \citep{bansal2023universal} extends this idea to arbitrary differentiable objectives. LGVF shares the principle of gradient-based steering but targets hard constraint satisfaction rather than soft conditional preferences and operates on flows rather than diffusion models.

\noindent \textbf{Constrained Generation.}
Several works address constraints in generative models. Reflected diffusion \citep{lou2023reflected} handles boundary constraints by reflecting the diffusion process at domain boundaries. Mirror diffusion \citep{liu2023mirror} uses mirror maps for simplex-constrained generation. \citet{fishman2023diffusion} incorporates manifold constraints for molecular generation. These approaches handle specific constraint types, whereas LGVF provides a framework for arbitrary differentiable constraints through a unified violation function formulation. Concurrent work on constrained diffusion \citep{khalafi2024constrained} explores Lagrangian methods for inequality constraints.

\begin{figure}[t]
    \centering
    \includegraphics[trim=0 120 0 0, clip, width=\textwidth]{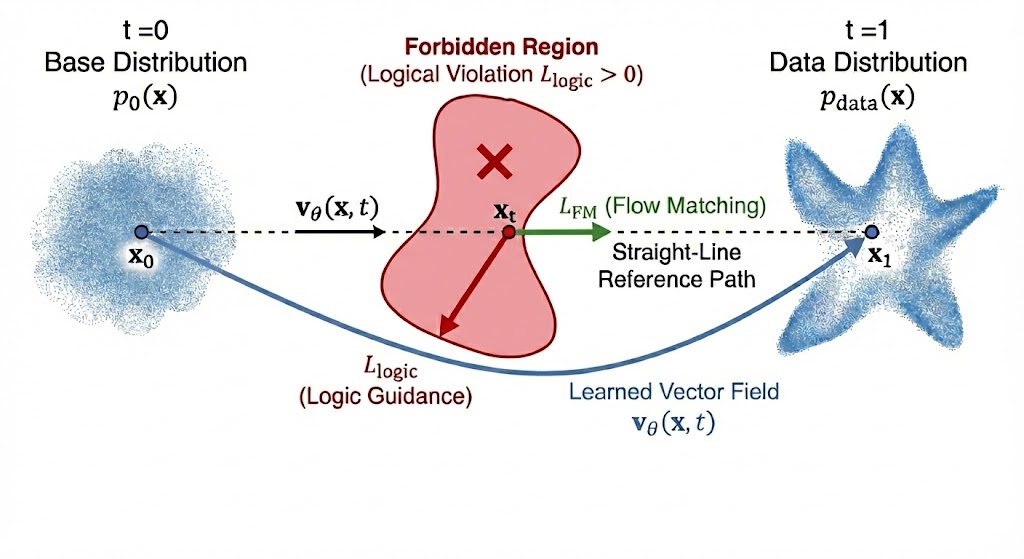}
    \caption{\textbf{LGVF learns constraint-aware vector fields.} 
    While standard flow matching follows straight-line paths (dashed) that may traverse forbidden regions, LGVF combines flow matching loss $\mathcal{L}_{\text{FM}}$ with logic guidance $\mathcal{L}_{\text{logic}}$ to learn vector fields that route trajectories around constraint violations. The logic loss penalizes violations at intermediate states $\mathbf{x}_t$, encouraging the model to discover valid paths from base distribution $p_0$ to target $p_{\text{data}}$.}
    \vspace{-10 mm}
    \label{fig:concept}
\end{figure}


\vspace{-4 mm}
\section{Methodology}
\label{sec:method}

We present Logic-Guided Vector Fields (LGVF), a framework for incorporating logical constraints into flow-based generative models. Our approach operates at two stages: a training-time loss that penalizes constraint violations along flow trajectories, and an inference-time adjustment that corrects the learned vector field during sampling. Let $p_{\text{data}}(x)$ denote a target distribution over $\mathbb{R}^d$ and let $\phi: \mathbb{R}^d \rightarrow \{\text{True}, \text{False}\}$ be a logical constraint that valid samples must satisfy. Our goal is to learn a generative model producing samples $x \sim p_\theta(x)$ such that $p_\theta$ approximates $p_{\text{data}}$ in the valid region while achieving $\mathbb{P}[\phi(x) = \text{True}] \approx 1$. We assume access to a differentiable relaxation of the constraint, expressed as a violation function $\ell_{\text{logic}}: \mathbb{R}^d \rightarrow \mathbb{R}_{\geq 0}$ satisfying $\ell_{\text{logic}}(x) = 0$ if and only if $\phi(x) = \text{True}$. The magnitude of $\ell_{\text{logic}}(x)$ reflects the degree of violation, and we require $\ell_{\text{logic}}$ to be differentiable almost everywhere to enable gradient-based optimization. Many constraints of practical interest admit natural relaxations: linear constraints $a^\top x \geq b$ become $\ell(x) = \max(0, b - a^\top x)$ and spherical constraints $\|x - c\| \leq r$ become $\ell(x) = \max(0, \|x-c\| - r)$. 


Flow matching \citep{lipman2023flow, liu2023flowstraight} learns a continuous-time flow that transforms a base distribution $p_0$ (typically $\mathcal{N}(0, I)$) into the target $p_1 = p_{\text{data}}$. The flow is defined by an ordinary differential equation $dx_t/dt = v_\theta(x_t, t)$ with $x_0 \sim p_0$ and $t \in [0, 1]$, where $v_\theta$ is a learnable vector field. Samples are generated by integrating the ODE from $t=0$ to $t=1$. Direct optimization of the marginal flow is intractable, so conditional flow matching regresses onto conditional vector fields. Given paired samples $(x_0, x_1)$ with $x_0 \sim p_0$ and $x_1 \sim p_{\text{data}}$, the conditional path is $x_t = (1 - t) x_0 + t x_1$ with conditional velocity $u_t = x_1 - x_0$. The flow matching objective minimizes
\begin{equation}
    \mathcal{L}_{\text{FM}}(\theta) = \mathbb{E}_{t \sim \mathcal{U}[0,1], x_0 \sim p_0, x_1 \sim p_{\text{data}}} \left[ \| v_\theta(x_t, t) - (x_1 - x_0) \|^2 \right].
    \label{eq:fm_loss}
\end{equation}
Standard flow matching provides no mechanism for constraint enforcement. Even if target samples satisfy $\phi(x_1) = \text{True}$, intermediate states $x_t$ may violate constraints, and the learned flow may route samples through forbidden regions.

\noindent \textbf{Logic-Guided Vector Fields.} Our key idea is that constraints should be enforced along the entire flow trajectory, not just at the endpoint. Figure~\ref{fig:concept} illustrates this idea. While standard flow matching follows the straight-line reference path from $x_0$ to $x_1$ (dashed line), this path may pass through forbidden regions where $\ell_{\text{logic}} > 0$. LGVF introduces a logic loss that penalizes violations at intermediate states $x_t$, creating a downward gradient (red arrow) that pushes the trajectory away from constraint violations. Combined with the flow matching objective (green arrow), the learned vector field $v_\theta(x, t)$ traces a curved path that avoids forbidden regions while still transporting samples to the target distribution. We augment flow matching with a time-weighted logic loss:
\begin{equation}
    \mathcal{L}_{\text{LGVF}}(\theta) = \underbrace{\mathcal{L}_{\text{FM}}(\theta)}_{\text{Flow Matching}} + \underbrace{\mathbb{E}_{t, x_0, x_1} \left[ \lambda(t) \cdot \ell_{\text{logic}}(x_t) \right]}_{\text{Logic Guidance}}
    \label{eq:lgvf_loss}
\end{equation}
where $x_t = (1-t)x_0 + tx_1$ and $\lambda(t) = \lambda_{\max} \cdot t^\alpha$ is a time-dependent weight.\footnote{We use $\alpha = 1$ (linear weighting) in all experiments.} As shown in Figure~\ref{fig:concept}, the flow matching term $\mathcal{L}_{\text{FM}}$ encourages the vector field to point toward the target $x_1$, while the logic guidance term penalizes trajectories that enter forbidden regions. The time-dependent weighting $\lambda(t)$ reflects the geometry of flow matching: at $t \approx 0$, samples lie near the base distribution where constraint violations are often unavoidable; at $t \approx 1$, samples approach the target where constraints must hold.Algorithm~\ref{alg:training} summarizes the training procedure.

\begin{algorithm}[t]
\caption{LGVF Training}
\label{alg:training}
\begin{algorithmic}[1]
\REQUIRE Dataset $\mathcal{D}$, base distribution $p_0$, violation function $\ell_{\text{logic}}$, batch size $B$, hyperparameters $\lambda_{\max}, \alpha$
\STATE Define weighting schedule $\lambda(t) \triangleq \lambda_{\max} t^{\alpha}$
\WHILE{not converged}
    \STATE Sample $\{x_1^{(i)}\}_{i=1}^B \sim \mathcal{D}$, $\{x_0^{(i)}\}_{i=1}^B \sim p_0$, $\{t^{(i)}\}_{i=1}^B \sim \mathrm{Unif}(0,1)$
    \STATE Compute $x_t^{(i)} = (1 - t^{(i)}) x_0^{(i)} + t^{(i)} x_1^{(i)}$
    \STATE $\mathcal{L}_{\text{FM}} \gets \frac{1}{B} \sum_{i=1}^B \| v_\theta(x_t^{(i)}, t^{(i)}) - (x_1^{(i)} - x_0^{(i)}) \|^2$
    \STATE $\mathcal{L}_{\text{logic}} \gets \frac{1}{B} \sum_{i=1}^B \lambda(t^{(i)}) \cdot \ell_{\text{logic}}(x_t^{(i)})$
    \STATE Update $\theta$ to minimize $\mathcal{L}_{\text{FM}} + \mathcal{L}_{\text{logic}}$ (e.g., Adam)
\ENDWHILE
\RETURN $v_\theta$
\end{algorithmic}
\end{algorithm}

\noindent \textbf{Logic-Adjusted Sampling.} Training-time guidance encourages constraint-aware vector fields but cannot guarantee zero violations, particularly for complex geometries. We complement LGVF training with inference-time adjustment that directly corrects the flow during sampling. During ODE integration, we modify the vector field by subtracting the constraint gradient:
\begin{equation}
    \tilde{v}(x_t, t) = v_\theta(x_t, t) - \eta(t) \cdot \nabla_x \ell_{\text{logic}}(x_t)
    \label{eq:adjusted}
\end{equation}
where $\eta(t) \geq 0$ controls adjustment strength. The correction term $-\nabla_x \ell_{\text{logic}}(x_t)$ points toward the direction of the steepest decrease in violation, pushing samples toward valid regions. We apply adjustment only at later times when samples approach the target:
\begin{equation}
    \eta(t) = \begin{cases}
        0 & \text{if } t \leq t_0 \\
        \eta_{\max} \cdot \left( \frac{t - t_0}{1 - t_0} \right)^2 & \text{if } t > t_0
    \end{cases}
\end{equation}
with threshold $t_0 = 0.3$. The quadratic ramp-up ensures smooth correction without abrupt changes that could destabilize integration. Algorithm~\ref{alg:sampling} presents the complete sampling procedure. The adjusted flow solves a modified ODE that simultaneously performs density transport (via $v_\theta$) and constraint optimization (via gradient descent on $\ell_{\text{logic}}$). When $\eta(t)$ is small relative to $\|v_\theta\|$, the adjustment acts as a perturbation that nudges samples toward validity without significantly altering the learned distribution.

\begin{algorithm}[t]
\caption{LGVF Sampling with Logic Adjustment}
\label{alg:sampling}
\begin{algorithmic}[1]
\REQUIRE Trained $v_\theta$, violation function $\ell_{\text{logic}}$, steps $K$, parameters $\eta_{\max}, t_0$ (with $0 \le t_0 < 1$)
\STATE Sample $x_0 \sim p_0$, set $x \gets x_0$, and $\Delta t \gets 1/K$
\FOR{$k = 0, \ldots, K-1$}
    \STATE $t \gets k\,\Delta t$
    \STATE $v \gets v_\theta(x, t)$
    \IF{$t > t_0$}
        \STATE $\eta \gets \eta_{\max} \cdot \left(\frac{t - t_0}{1 - t_0}\right)^2$
        \STATE $v \gets v - \eta \cdot \nabla_x \ell_{\text{logic}}(x)$ \COMMENT{gradient descent on violation}
    \ENDIF
    \STATE $x \gets x + \Delta t \cdot v$ \COMMENT{Euler step}
\ENDFOR
\RETURN $x$ \COMMENT{$\approx x_1$}
\end{algorithmic}
\end{algorithm}

\section{Experiments}
\label{sec:experiments}

We evaluated LGVF under three types of constraints of increasing complexity in 2D, followed by a scalability study in higher dimensions.

\noindent \textbf{Experimental Setup.} We parameterize the vector field $v_\theta(x, t)$ using a 3-layer MLP with 128 hidden units and ReLU activations. Time is concatenated directly with spatial coordinates as input. All models are trained with Adam using a learning rate $3 \times 10^{-3}$ for 8,000 iterations with a batch size of 256. For LGVF, we use time-dependent weighting $\lambda(t) = \lambda_{\max} \cdot t$ with $\lambda_{\max} \in [10, 15]$ depending on constraint difficulty. Samples are generated via Euler integration with 100 steps. Logic-adjusted sampling applies correction for $t > 0.3$ with strength $\eta(t) = \eta_{\max} \cdot ((t - 0.3)/0.7)^2$ where $\eta_{\max} \in [0.5, 1.5]$. We report three metrics computed over 2,000 generated samples: (1) \emph{Violation Rate}, the percentage of samples failing to satisfy the constraint; (2) \emph{Average Violation}, the mean constraint violation magnitude across all samples; and (3) \emph{MMD}, Maximum Mean Discrepancy with Gaussian kernel ($\sigma = 1.0$) measuring distributional fidelity to the target.

\noindent \textbf{Case Study 1: Linear Half-Plane Constraint.} Our first case study considers the linear constraint $\phi(x) \equiv (x_1 + x_2 \geq 0)$, partitioning $\mathbb{R}^2$ into valid and forbidden half-planes. The target distribution is a mixture of two Gaussians with means $\mu_1 = (-1.5, 2.0)$ and $\mu_2 = (2.0, 0.5)$, both in the valid region, with $\sigma = 0.4$. The differentiable violation function is $\ell_{\text{logic}}(x) = \max(0, -(x_1 + x_2))$.

Figure~\ref{fig:linear}(a) shows the target distribution with the forbidden region shaded. Flow Matching (b) achieves a 2.2\% violation rate, with violations occurring near the diagonal boundary between the two modes. LGVF (c) provides modest improvement to 2.0\%, while LGVF + Adjusted (d) reduces violations to 0.4\%, an 82\% reduction over the baseline. The violations in Flow Matching arise because some trajectories from the origin cut through the forbidden region when transporting mass to the upper-left mode. LGVF learns to bend these trajectories, and inference-time adjustment eliminates nearly all remaining violations.

\begin{figure}[t]
\centering
\includegraphics[width=\textwidth]{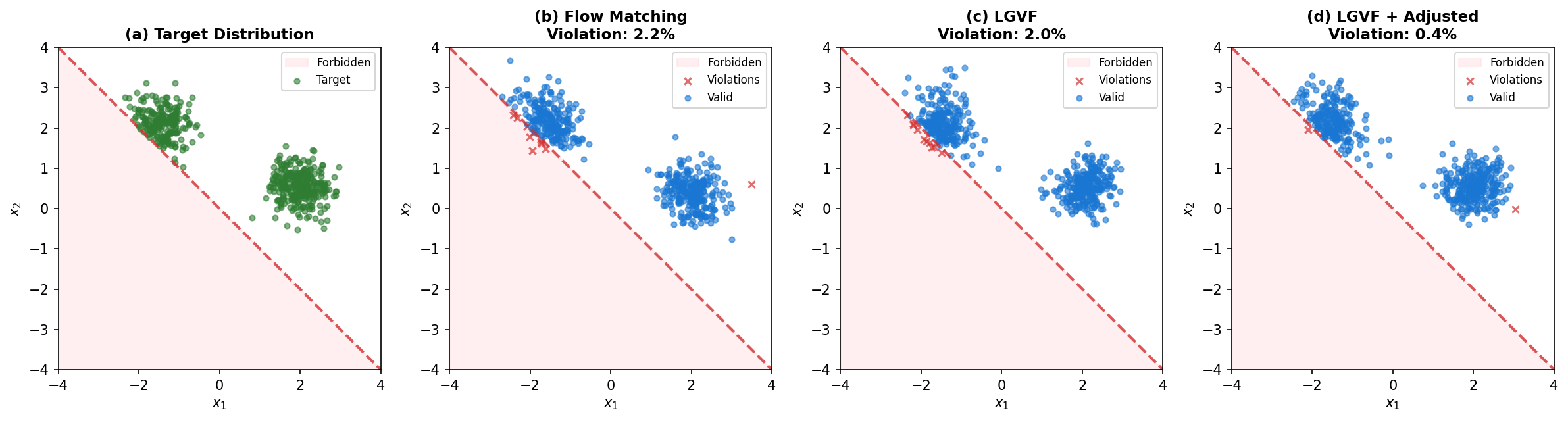}
\caption{\textbf{Case Study 1: Linear half-plane constraint} $(x_1 + x_2 \geq 0)$. (a) Target distribution with two Gaussian modes in the valid region. (b) Flow Matching: 2.2\% violations near the diagonal boundary. (c) LGVF: slight improvement to 2.0\%. (d) LGVF + Adjusted: 0.4\% violations (82\% reduction). Red crosses indicate violations; blue dots are valid samples.}
\vspace{-3 mm}
\label{fig:linear}
\end{figure}

\begin{figure}[t]
\centering
\includegraphics[width=\textwidth]{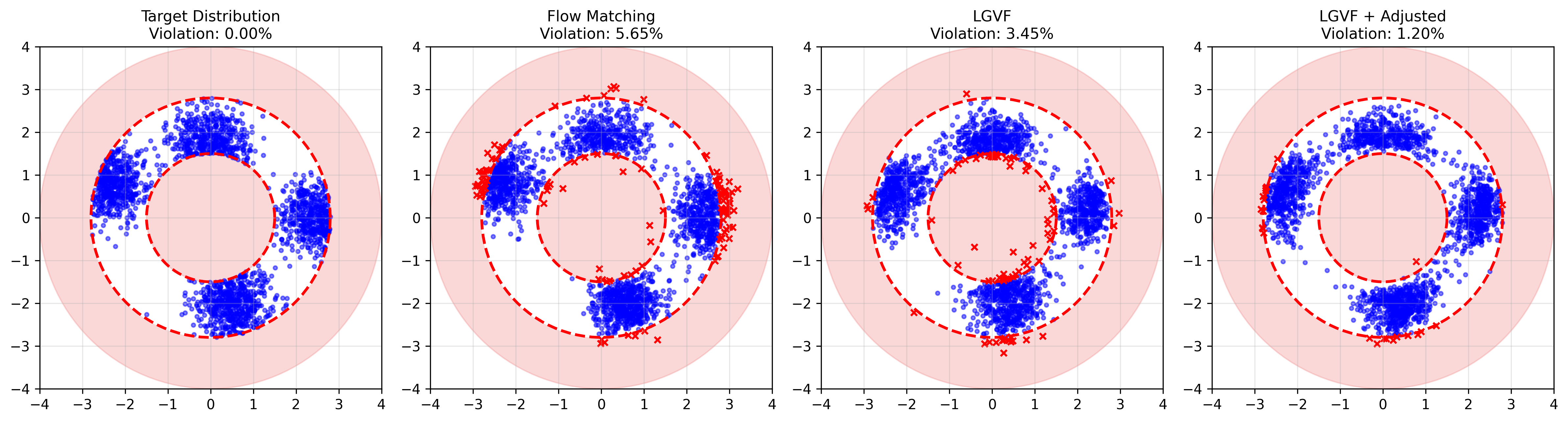}
\caption{\textbf{Case Study 2: Nonlinear ring constraint} $(1.5 \leq \|x\| \leq 2.8)$. The valid region is the white annulus; inner disk and outer region are forbidden. Flow Matching produces 5.65\% violations in both regions. LGVF reduces this to 3.45\%, and LGVF + Adjusted achieves 1.20\% (79\% reduction), with violations nearly eliminated at both boundaries.}
\vspace{-6 mm}
\label{fig:ring}
\end{figure}

\begin{figure}[t]
\centering
\includegraphics[width=\textwidth]{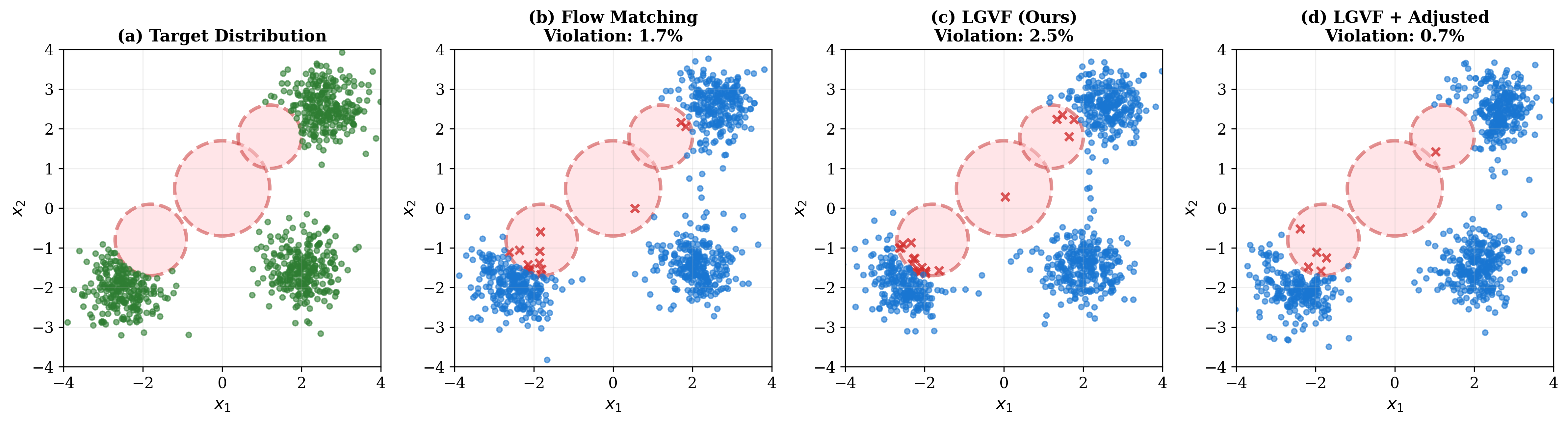}
\caption{\textbf{Case Study 3: Multi-obstacle avoidance.} Three circular obstacles (pink) block direct paths to target modes (green in panel a). Flow Matching: 1.7\% violations. LGVF alone increases violations to 2.5\% due to complex non-convex geometry, but LGVF + Adjusted recovers to 0.7\% (59\% improvement), demonstrating the robustness of inference-time correction.}
\vspace{-9 mm}
\label{fig:obstacles}
\end{figure}

\noindent \textbf{Case Study 2: Nonlinear Ring Constraint.} To evaluate LGVF on nonlinear constraints, we consider an annulus constraint $\phi(x) \equiv (r_{\min} \leq \|x\| \leq r_{\max})$ with $r_{\min} = 1.5$ and $r_{\max} = 2.8$. This constraint requires samples to lie within a ring of width 1.3, avoiding both the inner disk and the outer region. The target consists of four Gaussian modes positioned near the ring boundaries to create a challenging scenario, with $\sigma = 0.45$ chosen to induce boundary violations. The violation function is $\ell_{\text{logic}}(x) = \max(0, r_{\min} - \|x\|) + \max(0, \|x\| - r_{\max})$. Figure~\ref{fig:ring} shows the results. This case study is more challenging, with Flow Matching achieving a 5.65\% violation rate; violations occur both in the inner disk (samples that fail to reach a sufficient radius) and beyond the outer boundary (samples that overshoot). LGVF reduces violations to 3.45\% through training-time guidance, while LGVF + Adjusted achieves 1.20\%. The improvement is particularly visible in the outer boundary region, where the inference-time adjustment successfully pulls back samples that would otherwise exceed $r_{\max}$.

\noindent \textbf{Case Study 3: Multi-Obstacle Avoidance.} The third case study involves three circular obstacles, motivated by robotic path planning. The forbidden regions are: a central obstacle at $(0, 0.5)$ with a radius of 1.2, a lower-left obstacle at $(-1.8, -0.8)$ with a radius of 0.9, and an upper-right obstacle at $(1.2, 1.8)$ with a radius of 0.8. The target distribution consists of three Gaussian modes positioned such that straight-line paths from the origin intersect obstacles: $\mu_1 = (-2.5, -2.0)$, $\mu_2 = (2.5, 2.5)$, and $\mu_3 = (2.0, -1.5)$. Figure~\ref{fig:obstacles} presents the results. Flow Matching achieves a violation rate of 1.7\%, with violations distributed among all three obstacles. Interestingly, LGVF alone \emph{increases} violations to 2.5\%; we attribute this to the complex, non-convex loss landscape created by multiple disjoint forbidden regions, which can impede gradient-based training. However, inference-time adjustment successfully compensates: LGVF + Adjusted reduces violations to 0.7\%, a 59\% improvement over the baseline.

\begin{table}[b!]
\centering
\vspace{-5 mm}
\caption{\textbf{Summary of 2D results.} LGVF + Adjusted achieves the lowest violation rates across all constraints. Fidelity (MMD; $\times 10^{-3}$, lower is better) is preserved/improved in Linear and Ring, but degrades in the multi-obstacle setting, indicating a satisfaction--fidelity trade-off. Relative improvement is computed w.r.t.\ Flow Matching on Viol.\ (\%).}
\label{tab:results}
\vspace{0.5em}
\small
\begin{tabular}{llcccc}
\toprule
\textbf{Case Study} & \textbf{Method} & \textbf{Viol.\ (\%)} $\downarrow$ & \textbf{Avg.\ Viol.} $\downarrow$ & \textbf{MMD} $(\times 10^{-3})$ $\downarrow$ & \textbf{Improv.\ (Viol.\%)} \\
\midrule
\multirow{3}{*}{\shortstack[l]{1: Linear\\$(x_1 + x_2 \geq 0)$}}
& Flow Matching & 2.20 & 0.0044 & 0.89 & --- \\
& LGVF & 2.00 & 0.0038 & 0.99 & 9\% \\
& LGVF + Adjusted & \textbf{0.40} & \textbf{0.0008} & \textbf{0.29} & \textbf{82\%} \\
\midrule
\multirow{3}{*}{\shortstack[l]{2: Ring\\$(1.5 \leq \|x\| \leq 2.8)$}}
& Flow Matching & 5.65 & 0.0066 & 1.35 & --- \\
& LGVF & 3.45 & 0.0038 & 0.86 & 39\% \\
& LGVF + Adjusted & \textbf{1.20} & \textbf{0.0005} & \textbf{0.51} & \textbf{79\%} \\
\midrule
\multirow{3}{*}{\shortstack[l]{3: Obstacles\\(3 circles)}}
& Flow Matching & 1.70 & 0.0041 & \textbf{0.40} & --- \\
& LGVF & 2.50 & 0.0046 & 0.68 & $-47\%$ \\
& LGVF + Adjusted & \textbf{0.70} & \textbf{0.0023} & 0.82 & \textbf{59\%} \\
\bottomrule
\end{tabular}
\end{table}

\noindent \textbf{Case Study 4: High-Dimensional Scaling.} To evaluate whether LGVF scales beyond 2D, we test the linear constraint $\phi(x) \equiv (\sum_{i=1}^{d} x_i \geq 0)$ for dimensions $d \in \{10, 25, 50, 100\}$. The target distribution is a mixture of four Gaussians positioned in the valid half-space. We increase the hidden dimension to $\min(256, 4d)$ for higher-dimensional settings. Figure~\ref{fig:highdim} shows the results. As dimensionality increases, Flow Matching violations grow substantially, from 1.5\% at $d=10$ to 14.8\% at $d=100$, reflecting the increased difficulty of learning constraint-satisfying flows in high dimensions. LGVF alone provides variable improvement, effectively reducing violations at $d=100$ (14.8\% $\to$ 3.2\%) but less so at $d=50$ (9.1\% $\to$ 7.0\%). Specifically, LGVF + Adjusted maintains near-zero violation rates in all dimensions: 0.0\% at $d=10$ and $d=25$, 0.4\% at $d=50$, and 0.1\% at $d=100$. This shows that the inference-time gradient correction scales effectively, achieving 96–100\% violation reduction even as the baseline degrades with dimensionality.

Table~\ref{tab:results} summarizes the results in all 2D case studies. LGVF + Adjusted consistently achieves the lowest violation rates, with improvements ranging from 59\% to 82\% relative to Flow Matching. The training-time for LGVF alone provides benefits for convex constraint geometries (Case Studies 1 and 2) but can struggle with multi-region non-convex constraints (Case Study 3). The inference-time adjustment is uniformly beneficial, providing robust violation reduction regardless of constraint complexity. In particular, MMD scores remain comparable or improve with LGVF, indicating that there is no trade-off between constraint satisfaction and distributional fidelity. Combined with the high-dimensional results (Figure~\ref{fig:highdim}), these experiments demonstrate that LGVF provides reliable constraint enforcement on various geometries and scales effectively from 2D to 100D.

\begin{wrapfigure}{r}{0.52\textwidth}
  \centering
  \includegraphics[width=\linewidth]{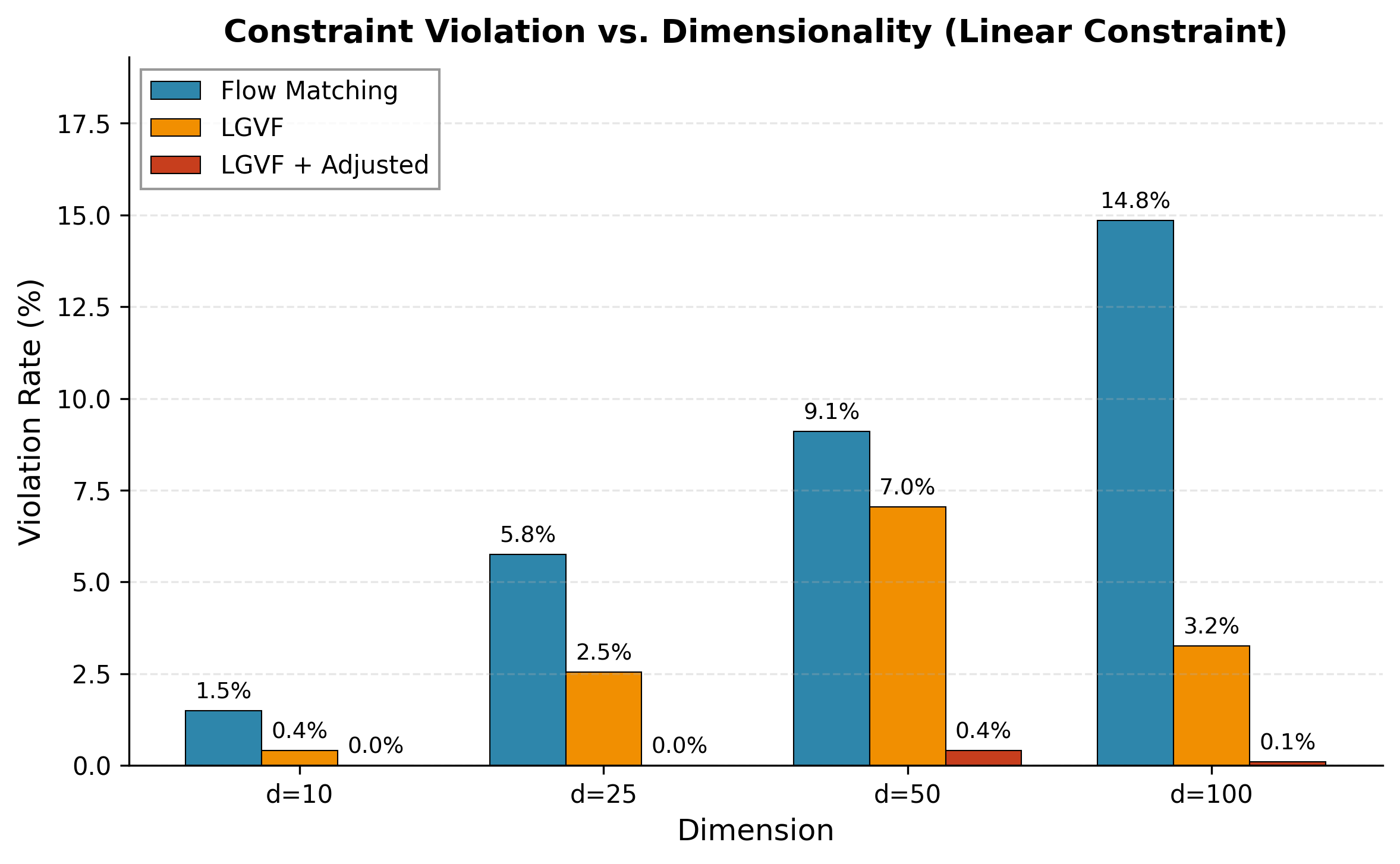}
  \vspace{-9 mm}
  \caption{Case Study 4: Scaling to high dimensions. Flow Matching violations increase with dimension, while LGVF + Adjusted maintains near-zero violations.}
  \label{fig:highdim}
\end{wrapfigure}

\noindent \textbf{Ablation Study.} We conduct ablation studies on the linear constraint case to understand hyperparameter sensitivity and component contributions. All experiments report the mean and standard deviation over 3 independent runs (Figure~\ref{fig:ablation}). Figure~\ref{fig:ablation}(a) shows the effect of training-time weight $\lambda_{\max}$. Violations decrease monotonically with stronger guidance: from 1.78\% at $\lambda_{\max}{=}0$ (no training-time logic loss) to 0.20\% at $\lambda_{\max}{=}10$, reaching zero at $\lambda_{\max}{\geq}20$. This confirms that the logic loss effectively teaches the vector field to avoid forbidden regions, with the method remaining robust even at high values ($\lambda_{\max}{=}50$).

Figure~\ref{fig:ablation}(b) compares inference-time adjustment strength $\eta_{\max}$ for both LGVF + Adjusted and FM + Adjusted (correction applied to standard flow matching). Both methods improve with stronger adjustment, but LGVF + Adjusted consistently outperforms FM + Adjusted at every setting, achieving zero violations at $\eta_{\max}{=}0.5$ compared to $\eta_{\max}{=}2.0$ for FM + Adjusted. This demonstrates that training-time and inference-time guidance are complementary: the learned constraint-aware vector field provides a better starting point for gradient correction. Figure~\ref{fig:ablation}(c) examines adjustment start time $t_0$. Earlier correction ($t_0 {\leq} 0.3$) achieves zero violations, while delaying to $t_0{=}0.9$ yields 0.07\%. The method is robust to this choice, and we use $t_0{=}0.3$ as default. Figure~\ref{fig:ablation}(d) isolates component contributions. From the FM baseline (0.50\%), adding training-time guidance alone (LGVF) reduces violations to 0.20\%, while inference-time adjustment alone (FM + Adj) achieves 0.10\%. Combining both (LGVF + Adj) yields 0.02\%, a 96\% improvement.

\begin{figure}[t]
    \centering
    \includegraphics[width=\textwidth]{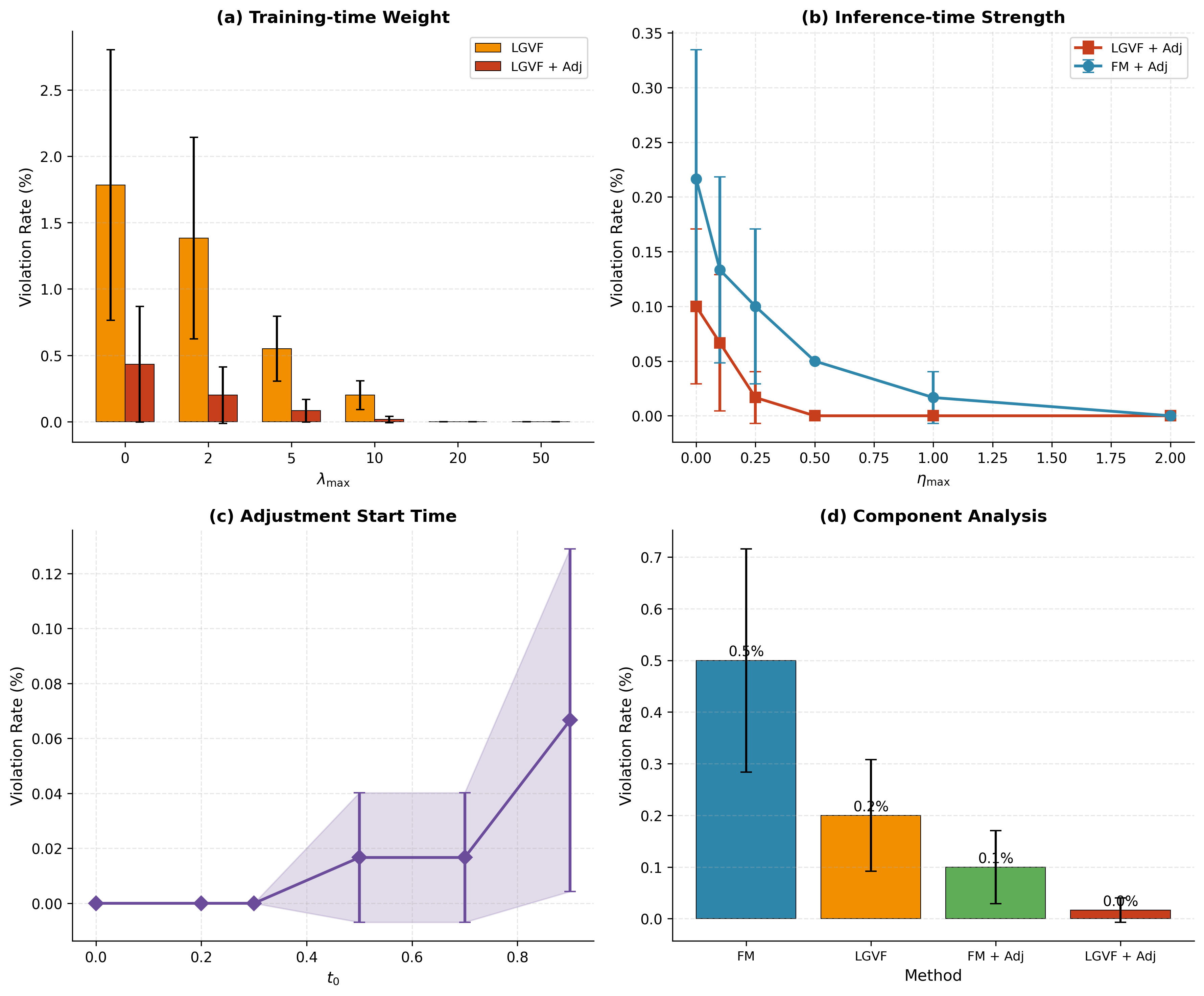}
    \caption{\textbf{Ablation studies on the linear constraint.} 
    (a) Training-time weight $\lambda_{\max}$: violations decrease monotonically with stronger guidance.
    (b) Inference-time strength $\eta_{\max}$: LGVF + Adj consistently outperforms FM + Adj across all settings.
    (c) Adjustment start time $t_0$: earlier correction is slightly more effective, but performance is robust.
    (d) Component analysis: combining training and inference guidance achieves the best results (0.02\% vs.\ 0.50\% baseline).
    Error bars show $\pm 1$ standard deviation over 3 runs.}
    \label{fig:ablation}
\vspace{-3 mm}
\end{figure}


\vspace{-4 mm}
\section{Conclusion}
\label{sec:conclusion}

We introduced a framework for incorporating logical constraints into flow-based generative models via training-time guidance and inference-time adjustment. By encouraging constraint satisfaction throughout the generative dynamics (rather than only at endpoints), LGVF learns constraint-aware vector fields. Across linear, nonlinear, and multi-region constraints in 2D, LGVF + Adjusted reduces violation rates by 59–82\% while preserving sample quality, with trajectory visualizations revealing emergent obstacle-avoidance behavior. We also observe strong scaling on a high-dimensional half-space constraint up to $d{=}100$, where LGVF + Adjusted achieves near-zero violations. Remaining challenges include extending to complex high-dimensional data (e.g., images and molecules), improving training-time robustness under highly non-convex constraint landscapes, and learning violation functions for implicit constraints alongside an analysis of the satisfaction–fidelity trade-off.

\newpage
\newpage
\bibliography{LGVF_cleaned_references}
\newpage
\appendix
\section{Theoretical Insights for Logic-Guided Vector Fields}
\label{app:theory}

This appendix provides supporting theory for the empirical behaviors observed in the main text:
(i) inference-time adjustment reliably reduces violation rates across constraint types;
(ii) late-time correction can preserve distributional fidelity; and (iii) multi-region constraints can
exhibit nonconvex ``potential field'' effects that make pure training-time guidance less reliable.

\subsection{Setup and assumptions}
Let $\phi:\mathbb{R}^d \to \{\mathrm{True},\mathrm{False}\}$ denote a logical constraint and
$\ell:\mathbb{R}^d\to\mathbb{R}_{\ge 0}$ a differentiable relaxation such that
$\ell(x)=0$ iff $\phi(x)=\mathrm{True}$. Define the feasible set $\mathcal{S}\doteq\{x:\ell(x)=0\}$.
Flow matching learns a time-dependent vector field $v_\theta(x,t)$ that induces the ODE
\begin{equation}
  \dot x_t = v_\theta(x_t,t), \qquad t\in[0,1].
  \label{eq:ode_base}
\end{equation}
LGVF uses inference-time logic adjustment (main text Eq.~(3)--(4)):
\begin{equation}
  \dot x_t = \tilde v(x_t,t)
  \doteq v_\theta(x_t,t) - \eta(t)\nabla \ell(x_t),
  \label{eq:ode_adjusted}
\end{equation}
where $\eta(t)\ge 0$ is a late-time schedule (e.g., $\eta(t)=0$ for $t\le t_0$). We use the following standard regularity conditions (local versions suffice on the sampled region).
\begin{itemize}
\item \textbf{(A1) Lipschitz drift.} $v_\theta(\cdot,t)$ is $L_v$-Lipschitz in $x$ for each $t$.
\item \textbf{(A2) Smooth violation.} $\ell$ is $L_\ell$-smooth on $\{x:\ell(x)>0\}$, i.e.,
$\|\nabla\ell(x)-\nabla\ell(y)\|\le L_\ell\|x-y\|$.
\item \textbf{(A3) Bounded gradients.} $\|\nabla\ell(x)\|\le G$ on the sampled region.
\end{itemize}
For hinge-type relaxations used in the experiments (half-spaces, rings, obstacles), $\ell$ is
differentiable a.e., and the results extend with subgradients.

\subsection{A Lyapunov view: why inference-time adjustment reduces violations}
The adjustment term in~\eqref{eq:ode_adjusted} performs gradient descent on $\ell$ while
transporting mass via $v_\theta$. The next lemma makes this explicit.

\begin{lemma}[Instantaneous violation rate under adjusted dynamics]
\label{lem:violation_derivative}
For almost every $t$ along a solution of~\eqref{eq:ode_adjusted},
\begin{equation}
  \frac{d}{dt}\ell(x_t)
  = \nabla\ell(x_t)^\top v_\theta(x_t,t)
  - \eta(t)\|\nabla\ell(x_t)\|^2.
  \label{eq:violation_derivative}
\end{equation}
\end{lemma}

\begin{proof}[Proof sketch]
By the chain rule, $\frac{d}{dt}\ell(x_t)=\nabla\ell(x_t)^\top \dot x_t$ a.e.
Substitute $\dot x_t=v_\theta(x_t,t)-\eta(t)\nabla\ell(x_t)$.
\end{proof}

Equation~\eqref{eq:violation_derivative} decomposes violation evolution into (i) an \emph{alignment}
term $\nabla\ell^\top v_\theta$ that can be positive or negative, and (ii) a \emph{dissipative} term
$-\eta\|\nabla\ell\|^2$ that is always non-positive. This explains why adjustment is uniformly helpful
in the experiments: it adds a guaranteed descent component on $\ell$.

\begin{proposition}[Sufficient condition for monotone violation decrease]
\label{prop:monotone_decrease}
Suppose that for all times when $\ell(x_t)>0$ we have
\begin{equation}
  \eta(t)\ \ge\
  \frac{\nabla\ell(x_t)^\top v_\theta(x_t,t)}{\|\nabla\ell(x_t)\|^2}.
  \label{eq:sufficient_eta}
\end{equation}
Then $\frac{d}{dt}\ell(x_t)\le 0$ whenever $\ell(x_t)>0$, i.e., violations do not increase over time.
\end{proposition}

\begin{proof}[Proof sketch]
Plug~\eqref{eq:sufficient_eta} into~\eqref{eq:violation_derivative}.
\end{proof}

\paragraph{Interpretation.}
Condition~\eqref{eq:sufficient_eta} states that the adjustment must dominate the component of
$v_\theta$ that points ``into'' infeasible regions (as measured by $\nabla\ell$). Late-time schedules
can satisfy this condition near $t=1$ even when it may fail early, consistent with the design
choice $\eta(t)=0$ for $t\le t_0$.

\subsection{Discrete-time (Euler) sampling: a one-step descent bound}
Algorithm~2 uses Euler integration with step size $\Delta t=1/K$:
\begin{equation}
  x_{k+1}
  = x_k + \Delta t\Big(v_\theta(x_k,t_k)-\eta_k\nabla\ell(x_k)\Big).
  \label{eq:euler_update}
\end{equation}

\begin{proposition}[One-step upper bound on violation]
\label{prop:one_step_bound}
Assume $\ell$ is $L_\ell$-smooth at $x_k$. Then
\begin{align}
  \ell(x_{k+1})
  \le\ &\ell(x_k)
  +\Delta t\,\nabla\ell(x_k)^\top v_\theta(x_k,t_k)
  -\eta_k\Delta t\,\|\nabla\ell(x_k)\|^2 \nonumber\\
  &\quad + \frac{L_\ell}{2}\Delta t^2\,
  \|v_\theta(x_k,t_k)-\eta_k\nabla\ell(x_k)\|^2.
  \label{eq:one_step_smooth}
\end{align}
\end{proposition}

\begin{proof}[Proof sketch]
Use smoothness: $\ell(y)\le \ell(x)+\nabla\ell(x)^\top(y-x)+\frac{L_\ell}{2}\|y-x\|^2$ with $y=x_{k+1}$.
\end{proof}

\paragraph{Consequence.}
For sufficiently small $\Delta t$ and/or sufficiently large $\eta_k$, the negative term
$-\eta_k\Delta t\|\nabla\ell(x_k)\|^2$ dominates the $O(\Delta t^2)$ remainder, yielding net
decrease in $\ell$ per step. This provides a direct explanation for the strong empirical reduction in
violati on rates under ``LGVF + Adjusted.''

\subsection{Geometric specialization to the paper's constraint families}
The hinge relaxations in the experiments admit especially transparent interpretations.

\paragraph{Half-space constraint.}
For $\phi(x)\equiv (a^\top x\ge b)$ and $\ell(x)=\max(0,b-a^\top x)$, when violated we have
$\nabla\ell(x)=-a$, hence the correction is $+\eta a$ pointing directly toward feasibility.
From Lemma~\ref{lem:violation_derivative}, whenever $a^\top x_t<b$,
\begin{equation}
  \frac{d}{dt}\ell(x_t)
  = -a^\top v_\theta(x_t,t) - \eta(t)\|a\|^2.
  \label{eq:halfspace_rate}
\end{equation}
If $a^\top v_\theta(x,t)\le M$ along trajectories, then choosing $\eta(t)\ge M/\|a\|^2+\delta$
implies $\frac{d}{dt}\ell(x_t)\le -\delta\|a\|^2<0$, yielding linear decrease in violation magnitude.
This aligns with the near-elimination of boundary violations in the linear case study.

\paragraph{Obstacle avoidance.}
For a forbidden ball $\|x-c\|<r$ with $\ell(x)=\max(0,r-\|x-c\|)$, when inside the obstacle
\begin{equation}
  \nabla\ell(x)= -\frac{x-c}{\|x-c\|},
  \qquad
  -\eta\nabla\ell(x)= \eta\,\frac{x-c}{\|x-c\|},
  \label{eq:obstacle_grad}
\end{equation}
so the adjustment is a purely radial ``repulsion'' term (scaled by $\eta$). This provides a direct
mechanistic account of the emergent avoidance behavior visualized in the paper.

\paragraph{Ring constraint.}
For $\phi(x)\equiv (r_{\min}\le \|x\|\le r_{\max})$ and
$\ell(x)=\max(0,r_{\min}-\|x\|)+\max(0,\|x\|-r_{\max})$, the gradient is radially outward when
$\|x\|<r_{\min}$ and radially inward when $\|x\|>r_{\max}$, i.e., the adjustment acts as a
\emph{radial clamp} that corrects inner/outer boundary violations, consistent with the ring study.

\subsection{Why late-time schedules preserve sample quality: a perturbation bound}
A recurring empirical pattern is that late-time correction reduces violations without noticeably
damaging distributional fidelity. A simple stability calculation formalizes this intuition. Let $x_t$ solve the base ODE~\eqref{eq:ode_base} and $\tilde x_t$ solve~\eqref{eq:ode_adjusted},
with $\eta(t)=0$ for $t\le t_0$ so that $x_{t_0}=\tilde x_{t_0}$.

\begin{proposition}[Trajectory deviation under late-time adjustment]
\label{prop:gronwall}
Under (A1)--(A3), for any $t\in[t_0,1]$,
\begin{equation}
  \|\tilde x_t - x_t\|
  \le \int_{t_0}^{t} e^{L_v(t-s)}\,\eta(s)\,G\,ds.
  \label{eq:gronwall}
\end{equation}
In particular, if $\eta$ is supported near $t=1$ and is bounded, the deviation at $t=1$ is controlled.
\end{proposition}

\begin{proof}[Proof sketch]
Let $\delta_t=\tilde x_t-x_t$. Then
$\dot\delta_t = v_\theta(\tilde x_t,t)-v_\theta(x_t,t)-\eta(t)\nabla\ell(\tilde x_t)$.
Apply Lipschitzness of $v_\theta$ and boundedness of $\nabla\ell$, then Gr\"onwall's inequality.
\end{proof}

\paragraph{Implication.}
Equation~\eqref{eq:gronwall} quantifies the design rationale for the schedule in Eq.~(4):
turning on correction only after $t_0$ and ramping smoothly keeps the corrected trajectory close to
the learned flow, while still yielding systematic descent on $\ell$ (Lemma~\ref{lem:violation_derivative}).

\subsection{Constraint satisfaction and MMD: support mismatch perspective}
The experiments report that MMD is comparable and can improve when violations are reduced.
A minimal formal support-mismatch observation explains why \emph{reducing mass in infeasible
regions cannot hurt matching in the limit of perfect modeling}. Let $q=p_{\mathrm{data}}$ and assume $q(\mathcal{S}^c)=0$. Consider a characteristic kernel
(e.g., Gaussian), for which $\mathrm{MMD}_k(p,q)=0$ iff $p=q$.

\begin{proposition}[Infeasible mass precludes zero MMD]
\label{prop:mmd_support}
If $q(\mathcal{S}^c)=0$ and $p(\mathcal{S}^c)>0$, then $\mathrm{MMD}_k(p,q)>0$ for any
characteristic kernel $k$. Equivalently, $\mathrm{MMD}_k(p,q)=0$ implies $p(\mathcal{S}^c)=0$.
\end{proposition}

\begin{proof}[Proof sketch]
For characteristic kernels, $\mathrm{MMD}_k(p,q)=0 \Leftrightarrow p=q$.
If $p(\mathcal{S}^c)>0$ while $q(\mathcal{S}^c)=0$, then $p\neq q$.
\end{proof}

\paragraph{Interpretation.}
When the data distribution lives on the valid region, any probability mass assigned to invalid
regions is necessarily ``wasted'' from the standpoint of distribution matching. Thus, decreasing
violation probability can reduce an unavoidable component of discrepancy (and may therefore
improve MMD), while the perturbation bound (Proposition~\ref{prop:gronwall}) explains why such
improvements can occur without major distortion of the learned distribution.

\subsection{Why multi-obstacle constraints can be harder: nonconvex ``potential field'' effects}
For conjunction constraints implemented by summing violations
$\ell(x)=\sum_{i=1}^m \ell_i(x)$ (main text), the correction direction becomes the sum of gradients
$\nabla\ell(x)=\sum_i \nabla\ell_i(x)$. With multiple disjoint obstacles, each $\nabla\ell_i$ can point
in a different direction, and cancellations can occur in free space.

\begin{remark}[Spurious equilibria under summed repulsions]
\label{rem:spurious}
Even when each $\ell_i$ is convex in a local neighborhood (e.g., hinge distance to an obstacle),
the sum $\ell=\sum_i \ell_i$ can have critical points away from obstacles where
$\sum_i \nabla\ell_i(x)\approx 0$. In such regions, the adjustment direction becomes small or poorly
aligned with any single ``escape'' direction, and interactions with the transport field $v_\theta$
may yield slow progress or locally undesirable flow shaping.
\end{remark}
This remark provides a theoretical lens on the obstacle study: multi-region constraints create a more
complex violation landscape, where purely training-time shaping may be brittle, while inference-time
adjustment still supplies direct descent on $\ell$ whenever violations are present
(Lemma~\ref{lem:violation_derivative}).

\noindent \textbf{Takeaway.} The results above formalize three high-level messages consistent with the empirical findings:
(i) the adjustment term is a Lyapunov-style descent mechanism for constraint violations;
(ii) late-time schedules control trajectory deviation and help preserve sample quality; and
(iii) multi-obstacle geometries can induce nonconvex gradient interactions, making the two-stage
strategy (training-time shaping + inference-time correction) robust across constraint types.

\end{document}